\documentclass{article}
 \usepackage{arxiv}
 \usepackage{natbib}
 \usepackage{authblk}

\usepackage[utf8]{inputenc} 
\usepackage[T1]{fontenc}    
\usepackage{hyperref}       
\usepackage{url}            
\usepackage{booktabs}       
\usepackage{amsfonts}       
\usepackage{nicefrac}       
\usepackage{microtype}      
\usepackage{algorithm2e}

\usepackage{bm}
\usepackage{amsmath}
\usepackage{amsthm}
\usepackage{xcolor}
\usepackage{wrapfig}

\usepackage{natbib}
\usepackage{caption}
\usepackage{subcaption}
\usepackage{graphicx}

\newtheorem{theorem}{Theorem}
\newtheorem{lemma}{Lemma}
\DeclareMathOperator{\sign}{sgn}

\newcommand{\T}{\mathsf{T}}

\newcommand{\matr}[1]{\bm{#1}}
\newcommand{\mK}{\matr{K}}
\newcommand{\mX}{\matr{X}}

\newcommand{\mOmega}{\matr{\Omega}}
\newcommand{\mPmOmega}{{\matr{P}_{\matr{\Omega}}}}
\newcommand{\mRmOmega}{{\matr{R}_{\matr{\Omega}}}}
\newcommand{\mQmOmega}{{\matr{Q}_{\matr{\Omega}}}}

\newcommand{\mI}{\matr{I}}
\newcommand{\mZero}{\matr{0}}

\renewcommand{\vec}[1]{\bm{#1}}

\newcommand{\vx}{\vec{x}}
\newcommand{\vw}{\vec{w}}
\newcommand{\vxd}{\vw}

\newcommand{\vy}{\vec{y}}

\newcommand{\vomega}{\vec{\omega}}

\newcommand{\bigO}{\mathcal{O}}

\newcommand{\DDDD}{\ensuremath{\mathrm{D}_4}}

\title{Decision-Directed Data Decomposition}

\author[1]{Brent D. Davis*} \author[2,3]{Ethan C. Jackson} \author[1,3,4]{Daniel J. Lizotte}

\affil[1]{Department of Computer Science, Western University, London, ON~~N6A 5B7, Canada}
\affil[2]{School of Engineering, University of Guelph, Guelph, ON~~N1G 2W1, Canada}
\affil[3]{Vector Institute, Toronto, ON~~M5G 1M1, Canada}
\affil[4]{Department of Biostatistics \& Epidemiology, Schulich School of Medicine and Dentistry, Western University, London, ON N6A 5C1, Canada}

\begin{document}

\maketitle

\begin{abstract}
We present an algorithm, Decision-Directed Data Decomposition (\DDDD{}), which decomposes a dataset into two components. The first contains most of the useful information for a specified supervised learning task. The second orthogonal component contains little information about the task but retains associations and information that were not targeted. The algorithm is simple and scalable. We illustrate its application in image and text processing domains. Our results show that 1) post-hoc application of \DDDD{} to an image representation space can remove information about specified concepts without impacting other concepts, 2) \DDDD{} is able to improve predictive generalization in certain settings, and 3) applying \DDDD{} to word embedding representations produces state-of-the-art results in debiasing.
\end{abstract}

\section{Introduction}

Distributed feature representations of complex entities learned from data are useful for many tasks. For example, image representations from deep learning models have found many uses outside of the task they were originally trained on \citep{Gatys}, and word embeddings are used for supervised learning tasks, data exploration, and sense-making tasks in a variety of domains \citep{Dai2017}. However, such representations can carry information that is undesirable, either because it reflects undesirable bias (e.g.\ gender bias in word embeddings \citep{Caliskan2017}) or because it obfuscates other information that is relevant to the task at hand \citep{Goodfellow}, which can impact both data exploration and ability to generalize.

We introduce Data-Directed Data Decomposition (\DDDD{}), a technique to decompose a data matrix into two components. One component contains information about a specified classification or regression target that a linear model can use for prediction, while the other does not contain such information. It is this second component, orthogonal to the first, that is useful for further analyses. We will demonstrate that it excludes information about the specified supervised learning target, and that information about unrelated supervised learning targets is not affected.

Our main goal is to enable post-hoc removal of target concepts from data, which is useful when debiasing data. Debiasing is worthy of particular attention as bias can lead to unfair models and is notoriously difficult to dis-entrench \citep{Gonen2019}. However, we also present results that illuminate when \DDDD{} can improve generalization performance. Hence we focus our attention on \DDDD{}'s ability to remove information and bias, but we also consider applications that would normally be served by adversarial learning, where it is important to remove the ability to learn certain concepts from a representation with the goal of improving generalization \citep{Ribeiro2016}.

We describe our algorithm in detail in Section~\ref{ss:DDDD}. We identify uses of \DDDD{} and provide illustrative experimental examples in Section~\ref{ss:examples}, including state-of-the-art results on word embedding debiasing. In Section~\ref{ss:related} we discuss connections to related methodology and formalize additional properties of $\DDDD{}$ to describe the connections. Finally in Section~\ref{ss:conclusion} we conclude and identify future directions of research.

\section{Decision-Directed Data Decomposition}\label{ss:DDDD}

Our approach uses generalized linear supervised learning methods whose decision functions are of the form $h(\vx) = g(\vx^\T \vw)$, where $\vw$ is learned from labelled data and represents a direction in feature space that is most useful for predicting a target $\vy$, according to the loss function of the learner. (E.g.\ cross-entropy for logistic regression, hinge for SVM.) \DDDD{} will find these most useful directions and then project the data onto their orthogonal complement to create a new dataset with which we are \textit{not} able to predict the target well. The resulting data can then be used subsequently for analyses where learners should \textit{not} make use of the target concepts---whether explicitly or implicitly---in order to label future instances. \DDDD{} is presented in  Algorithm~\ref{alg:primal}; we present the relevant background and intuition here.

For a $p \times 1$ unit vector $\vomega$, the projection of the rows of a matrix $\mX$ onto $\vomega$ is given by $\mX_\parallel = \mX\vomega\vomega^\T$, and the projection onto its orthogonal complement is given by $\mX_\perp = \mX(\mI - \vomega\vomega^\T)$. 
\\For example, if  
$\mX = \left[\begin{array}{ccc}
1 & 0 & 1\\
0 & 1 & 1\\
1 & 0 & 0\\
0 & 1 & 0\\
\end{array}\right]$ and $\vomega = \left[ \begin{array}{c} 0 \\ 0 \\ 1 \end{array} \right]$, then
$\mX_\parallel = \left[\begin{array}{ccc}
0 & 0 & 1\\
0 & 0 & 1\\
0 & 0 & 0\\
0 & 0 & 0\\
\end{array}\right]$ and $\mX_\perp = \left[\begin{array}{ccc}
1 & 0 & 0\\
0 & 1 & 0\\
1 & 0 & 0\\
0 & 1 & 0\\
\end{array}\right]$.

Note $\mX = \mX_\parallel + \mX_\perp$, $\mX_\parallel \mX_\perp^\T = \mX_\perp \mX_\parallel^\T = \mZero$, and $\mX_\perp\vomega = \mZero$; hence if we consider the rows of $\mX_\perp$ as points in space, they have zero variability in the direction of $\vomega$; in other words, all information about where the points lie in the direction of $\vomega$ has been removed and therefore $\mX_\perp$ could be used in future analyses where that direction should be excluded from decision-making. 

In practice, it is unlikely that in a distributed representation only one direction contains information about a given target. Hence, we take the $\mX_\perp$ resulting from the first projection and remove the next best decision-direction, resulting in a new $\mX_\perp$, and so on. Continuing this process eventually gives $\mX_\parallel = \mX$ and $\mX_\perp = \mZero$, which obviously contains no information about the target (or about anything else). At any step along the way, we have removed some of the information about $\vy$ from $\mX$ that can be recovered by (generalized) linear learners, and in practice the quantity that remains can be reduced to zero. 


The complete \DDDD{} algorithm works as follows. 
Let $\mX$ be an $n \times p$ matrix of feature vectors, each of length $p$, and let $\vy$ be an $n \times 1$ vector of supervised learning targets. Let $\vxd$ be a $p \times 1$ decision vector learned from $\mX$ and $\vy$, and let $\vomega = \vxd/||\vxd||$. The projection of the rows of $\mX$ onto the space orthogonal to $\vomega$ is given by $\mX_\perp = \mX(\mI  - \vomega\vomega^\T).$ For all feature vectors $\vx_{\perp i}$ $i \in 1..n$ which correspond to the rows of $\mX_\perp$, we have $\vx_{\perp i}^\T \vomega = 0$. We note the following simplification of sequential orthogonal projections.

\begin{lemma}[Sequences of orthogonal projections]\label{lm:orthogonal-primal}
If for all $\vomega^{(i)}$, $\vomega^{(j)}$ in $\vomega^{(1)}, \vomega^{(2)}, ..., \vomega^{(p)}$ we have $\vomega^{(i)\T} \vomega^{(j)} = 0$, then $\mX \prod_i (\mI  - \vomega^{(i)}\vomega^{(i)\T}) = \mX (\mI  - \sum_i \vomega^{(i)}\vomega^{(i)\T}) $.
\end{lemma}

Using this lemma, we define $\mOmega^{(i)} \leftarrow \mI - \sum_{j = 1}^{i} \vomega^{(j)}\vomega^{(j)\T}$, which is the projection onto the space orthogonal to all of $\vomega^{(1)}$ through $\vomega^{(i)}$. This allows us to define $\mX^{(i)}_\perp = \mX \mOmega^{(i)}$ and $\mX^{(i)}_\parallel = \mX - \mX^{(i)}_\perp$. Our learner can then use $\mX^{(i)}_\perp$ and $\vy$ to identify the next direction to remove, and so on.

The rank of $\mOmega^{(i)}$ is $p - i$, and the rank of $\mX^{(i)}_\perp$ is also $p - i$ assuming $\mX$ had full rank to begin with. If the learning algorithm to be used with \DDDD{} requires a full-rank feature matrix, we can use Gram-Schmidt orthogonalization to produce an equivalent full-rank representation as follows. Create a matrix $\matr{G} = \left[ \vomega^{(1)},...,\vomega^{(i)} | \vec{\psi}^{(i+1)}, ..., \vec{\psi}^{p} \right]$, choosing the $\vec{\psi}$ so that $G$ has full rank. Perform (possibly modified) Gram-Schmidt orthogonalization
on $\matr{G}$. Since the $\vomega^{(1)},...,\vomega^{(i)}$ are already orthonormal they will be unchanged, and the remaining columns will form an orthonormal basis for their orthogonal complement; call those columns $\mPmOmega$. The learner can then use $\mX\mPmOmega$ and $\vy$ to learn the next direction $\tilde\vxd$ in $p - i$ dimensions, and then project that weight vector up to $p$ dimensions to obtain $\vw = \mPmOmega^\T \tilde{\vxd}$. Even if the learner does not require a full-rank input, the orthogonalization process may be desirable for numerical stability with large $p$.

If a full-rank feature matrix is not needed, then the time cost per iteration is $\bigO(p^2)$ to form $\mOmega$ (if it is updated in-place from the previous iteration) and $\bigO (n p^2)$ to project $X$, plus the cost of learning. If a full-rank feature matrix is needed, then the time cost per iteration is $\bigO(p^3)$ to form $\mPmOmega$ by Gram-Schmidt or QR and $\bigO(n (p-i)p)$ to project $X$, plus the cost of learning. In both cases, the space complexity (additional to storage of $\mX$ and $\vy$) is $O(p^2)$ to store the $\vomega^{(i)}$.

\begin{algorithm}
 \DontPrintSemicolon
 \KwData{Feature matrix $\mX$ ($n \times p$) of training points, targets $\vy$  ($n \times 1$).}
 \KwResult{Orthogonal basis vectors $\vomega^{(1)}, \vomega^{(2)}, ..., \vomega^{(p)}$}
 \For{$i$ from $1$ to $p$}{
    \eIf{learner does not require a full-rank feature matrix}{ 
        $\mOmega \leftarrow \mI - \sum_{j = 1}^{i - 1} \vomega^{(j)}\vomega^{(j)\T}$\;
        $\vxd \leftarrow \mathrm{learn}(\mX \mOmega, \vy)$\;
    }{
      $\mQmOmega, \mRmOmega \leftarrow \mathrm{QR}(\left[ \vomega^{(1)},...,\vomega^{(i-1)} | \vec{\psi}^{(i)}, ..., \vec{\psi}^{p} \right])$\;
    $\mPmOmega \leftarrow \mbox{last $p - i$ columns of }\mQmOmega$\;
    $ \tilde{\vxd} \leftarrow \mathrm{learn}(\mX \mPmOmega, \vy)$\;
    $ \vxd \leftarrow \mPmOmega^\T \tilde{\vxd}$\;
    }
  $\vomega^{(i)} \leftarrow \vxd / ||\vxd||$\;
  }
 \caption{\label{alg:primal} \DDDD{} Algorithm - Feature Representation}
\end{algorithm}
\newpage
\section{Application Examples}\label{ss:examples}

We now present three examples of how \DDDD{} can be applied. In Section~\ref{ss:removing}, we use an image processing example to show that \DDDD{} is able to remove information about a specified target concept without interfering with other tasks. In Section~\ref{ss:improving}, we show when and how \DDDD{} can lead to improved generalization in supervised learning. In Section~\ref{ss:debiasing}, we show how \DDDD{} can be used for debiasing of word embeddings, providing state-of-the-art results. 
\subsection{Removing a Target Concept}\label{ss:removing}

\begin{figure}[!ht]
    \centering
    \includegraphics[width=1.0\linewidth]{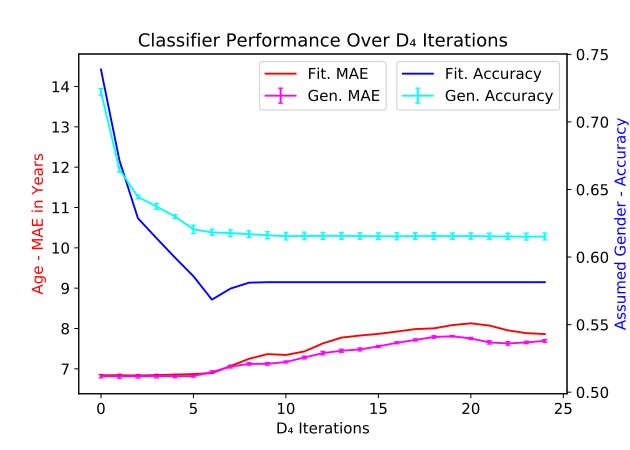}
    \caption{Performance of linear classifiers (regularized least-squares support vector machines with $\alpha=5\times 10^{3}$) over 25 iterative \DDDD{} projections onto linear decision boundary computed using IMDB gender labels. Data withheld from training the neural network model were split into: fitting ($\mathbf{fit}$, $n=22\,579$) and generalization estimation ($\mathbf{gen}, n=45\,158$), which was further split into $\mathbf{gen}_1 \dots  \mathbf{gen}_5$ where each $\mathbf{gen}_i$ is a $22\,579$ random sample from $\mathbf{gen}$ with replacement. At each iteration, a linear decision boundary using gender labels ($\mathbf{DB_G}$) is computed using $\mathbf{fit}$. Then, the data in both $\mathbf{fit}$ and $\mathbf{gen}$ are projected onto $\mathbf{DB_G}$, and classifiers are re-fitted for each pair from $\lbrace \mathbf{fit}, \mathbf{gen}_1, \dots , \mathbf{gen}_5 \rbrace \times  \lbrace \text{age}, \text{gender} \rbrace$. Age classifiers are evaluated using mean absolute error, and are shown in red and magenta for $\mathbf{fit}$ and $\mathbf{gen}$, respectively. Gender accuracy, according to the labels, is shown in blue and cyan for $\mathbf{fit}$ and $\mathbf{gen}$, respectively. Results for $\mathbf{gen}$ classifiers are reported as the mean $\pm$ standard deviation over $\mathbf{gen}_1, \dots , \mathbf{gen}_5$. Over $\mathbf{fit}$, stratified random prediction of age has a baseline error of $14.4$ years and stratified random prediction of gender has a baseline accuracy of $52\%$. \textbf{Iteration 0}: Before any \DDDD{} projections, age prediction errors for $\mathbf{fit}$ and $\mathbf{gen}$ are $6.84$ and $6.81 \pm 0.04$ years, respectively, and gender accuracy for $\mathbf{fit}$ and $\mathbf{gen}$ are $73.9\%$ and $72.8\% \pm 0.2\%$, respectively. \textbf{Iteration 5}: Gender classification accuracy sharply decreased to $56.8\%$ and $62.1\% \pm 0.3\%$ for $\mathbf{fit}$ and $\mathbf{gen}$, respectively, while age prediction performance did not change. Subsequent iterations resulted in increasing error for age prediction, and in gender accuracy convergence to $58.1\%$ and $61.4\% \pm 0.2\%$ for $\mathbf{fit}$ and $\mathbf{gen}$, respectively.}
    \label{fig:age_gender}
\end{figure}

Our first example demonstrates that \DDDD{} can remove information about one target in a neural representation space without sacrificing classification accuracy on other targets. Using the Deep Expectation of Apparent Age (DEX) method and accompanying IMDB dataset \citep{Rothe-IJCV-2016}, we trained a deep neural network to predict the age of human faces in images. The DEX method retrains a VGG16 model \citep{simonyan2014very} pre-trained on ImageNet \citep{deng2009imagenet} to predict age instead of the usual ImageNet classes. After training for $10$ epochs over $361\,246$ images, age prediction over $45\,158$ validation images not seen during training had a mean absolute error of $6.84$ years and standard deviation of $8.62$ years. Note that this result is not state-of-the-art, but is significantly better than that of a stratified random classifier (mean absolute error of $14.41$ years and standard deviation of $11.52$ years), and serves our purposes.

The IMDB dataset also includes binary gender labels. This labelling scheme reflects a simplified interpretation of gender identity that may constitute an undesirable bias. \DDDD{} can be used to iteratively decompose the neural representations of images into target concept and bias components so that separability on the bias component can be quantified and reduced.

Using the trained neural network, $22\,579$ projection fitting images and $45\,158$ generalization estimation images were transformed to points in the model's 4096-dimensional final internal representation space. As expected, linear learners can learn to predict age to a similar degree as observed while training the neural network. Perhaps less expected, the same points were also highly linearly separable on the target concept of binarized gender as it is assumed in the dataset -- see Figure \ref{fig:age_gender}, Iteration 0. Given that the representation space learned by the neural network for age prediction was also suitable for binarized gender prediction, we might be led to prematurely infer that the two concepts are necessarily dependent or entangled. \DDDD{} can be applied to challenge this inference.

After 5 iterations of \DDDD{}, the data were transformed such that separability according to the gender labels was sharply reduced without impacting predictability of age -- see Figure \ref{fig:age_gender}, Iteration 5. From this, we can infer that age prediction is much less dependent on separability of the gender labels than would be presumed without applying \DDDD{}. In other words, Figure \ref{fig:age_gender} shows a lower bound for the degree to which the target concept of age can be effectively disentangled from binary gender labels used in the IMDB dataset.

\DDDD{} is a versatile tool for target concept removal that can be applied post-hoc to data in neural representation spaces. 

\subsection{Improving Generalization}\label{ss:improving}

Our second example illustrates how the application of \DDDD{} can improve generalization error in particular settings by enforcing invariants. It has been long established that in order to achieve generalizability, predictive models must be invariant to features or concepts that are correlated with the specified target in the training set but that may be uncorrelated or even anti-correlated in other settings. As an extreme example, \cite{Ribeiro2016} constructed a synthetic setting where a classifier is trained to distinguish images of wolves from images of huskies, but where the wolves are only shown in snow and the huskies are never shown in snow. The resulting classifier is able to achieve 100\% training accuracy, but has no ability to generalize when presented with huskies in snow or wolves not in snow because its predictions are driven entirely by the presence of snow in the image.

We create a prototypical example to both illustrate this effect and demonstrate how \DDDD{} can mitigate it. 
Consider a dataset with $n = 100000$ and $p = 300$. We generate two random orthogonal directions, $\vw^*_1$ and $\vw^*_2$, in this space, which define two targets $y_1(\vx) = \varepsilon \sign \vx^\T \vw^*_1$ and $y_2(\vx) = \varepsilon \sign \vx^\T \vw^*_2$, where $P(\varepsilon = 1) = 0.9$ and $P(\varepsilon = -1) = 0.1$. We generate $n=100000$ multivariate normal feature vectors such that the correlation between $\vx^\T \vw^*_1$ and $\vx^\T \vw^*_2$ is $0.9$ and the standard deviations of $\vx^\T \vw^*_1$ and $\vx^\T \vw^*_2$ are 1 and 2, respectively, then we generate the labels $\vy_1$ and $\vy_2$. This effectively makes the signal in $\vw^*_2$ ``stronger'' than that in $\vw^*_1$ for  linear classifiers that have a prior that prefers small weights, i.e.\ that are regularized. We generate a test set that is the same in all respects except that the correlation is $-0.9$.
\begin{table*}[t]
    \caption{Performance on task defined by $\vy_1$ before and after using \DDDD{} to remove information about the task defined by $\vy_2$. ``Iteration 0'' refers to classifiers constructed using the original data. Iteration 1 shows results after one iteration of \DDDD{}. ``Loadings'' give the dot product between the learned classifier weights (normalized) and the weight vector used to define the decision boundary.}\label{tab:generalize}
    \centering
    \begin{tabular}{cccccc}
    \hline
         Iteration & Target & Train Accuracy & Test Accuracy & Weight on $w^*_1$ & Weight on $w^*_2$  \\
         \hline
         \hline
         0 & $\vy_1$ & 0.81 & 0.26 & 0.54 & 1.68 \\
           & $\vy_2$ & 0.88 & 0.87 & 0.39 & 1.83 \\
         \hline
         1 & $\vy_1$ & 0.61 & 0.82 & 0.84 &-0.69 \\
           & $\vy_2$ & 0.52 & 0.27 & 0.64 &-0.48 \\
         \hline
    \end{tabular}
\end{table*}

First, we train ridge logistic regression classifiers with $\lambda = 1$ on both $\vy_1$ and $\vy_2$. This achieves good training error for both, and good test error for $\vy_2$, but very poor test error for $\vy_1$. This is because the classifier for $\vy_1$ is mostly using $\vw^*_2$ to discriminate; in the training data both $\vw^*_1$ and $\vw^*_2$ are good for discriminating $\vy_1$, but this is not the case in the test data where the correlation has been reversed. After we apply one iteration of \DDDD{} and use the resulting data to train new classifiers, the test accuracy for $\vy_1$ jumps from $0.26$ to $0.82$, while the test accuracy for $\vy_2$ falls from $0.87$ to $0.27$. Note that the training error for $\vy_1$ actually falls from $0.81$ to $0.61$, as removal of the $\vw^*_2$ component makes fitting the regularized logistic regression more difficult, despite the improved test error. Table~\ref{tab:generalize} summarizes the results, and shows the loadings of weight vectors of each classifier onto $\vw^*_1$ and $\vw^*_2$, to illustrate the directions used by the classifiers.

\subsection{Debiasing}\label{ss:debiasing}


Our third example demonstrates how target concept removal can be applied to reduce representation bias in neural word embeddings. Recent work in debiasing word embeddings has shown that analogy tasks can reveal problematic biases in models learned from real world text \citep{Bolukbasi2016}. We illustrate the application of \DDDD{} in this space, comparing our results to two established debiasing methods: Bolukbasi et al.,'s approach, which we denote HARD-DEBIAS, and GN-GloVe, by \cite{Zhao2018}. Both approaches are successful in mitigating some bias while preserving the functional aspects of the word embeddings. However, close examination of more deeply ingrained biases by \cite{Gonen2019} (henceforth, GG) revealed that some popular debiasing methods are limited by the degree to which biases can be effectively removed. Hence, we task \DDDD{} to more deeply dis-entrench gender information from word embeddings.

\begin{figure*}[t]
    \centering
    \begin{subfigure}{.5\textwidth}
        \centering
        \includegraphics[width=1.0\linewidth]{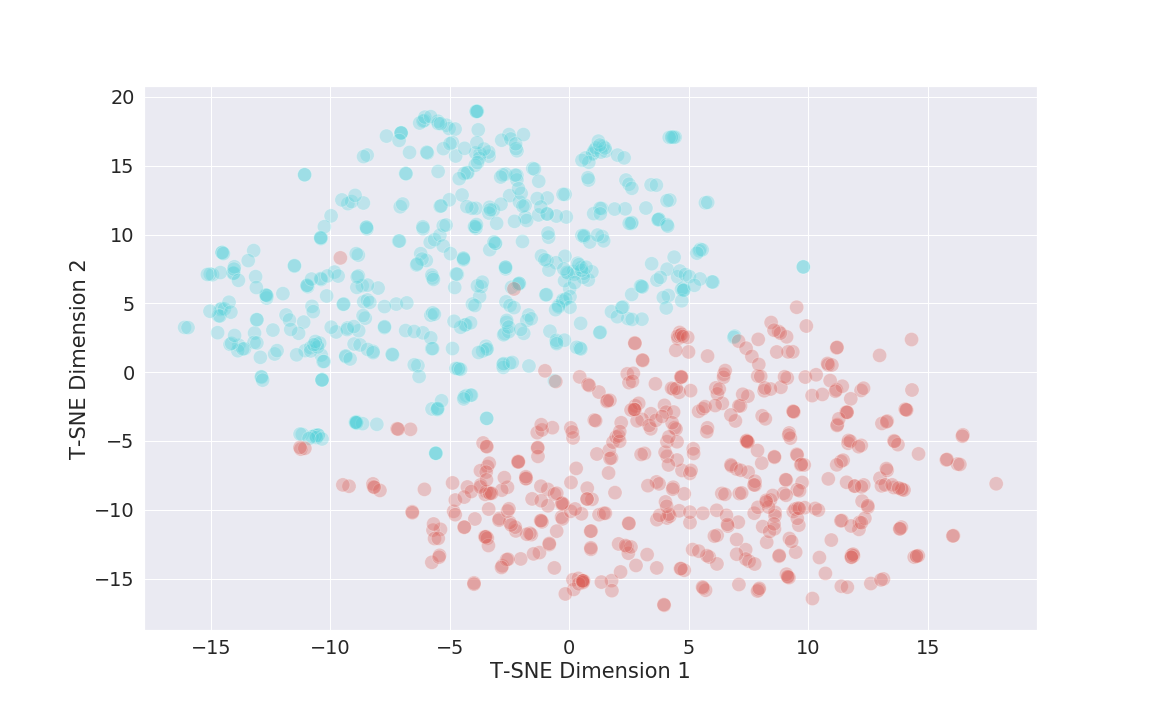}
        \caption{Before Debiasing}
        \label{fig:sub1}
    \end{subfigure}%
    \begin{subfigure}{.5\textwidth}
        \centering
        \includegraphics[width=1.0\linewidth]{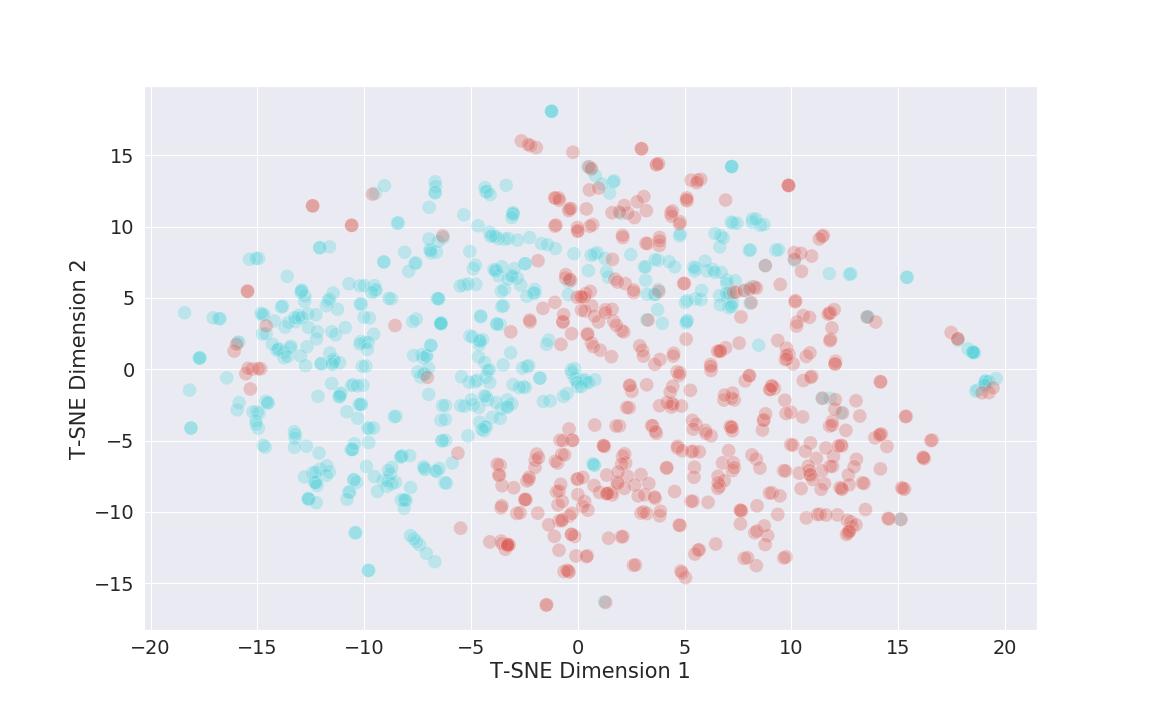}
        \caption{After Debiasing with 6 \DDDD{} projections}
        \label{fig:sub2}
        
    \end{subfigure}
    \caption{T-SNE \citep{VanDerMaaten2008} (\emph{perplexity}=40, \emph{n\_iterations}=300) representation of gendered word vectors before (a) and after 6 iterations of \DDDD{} (b). Points are coloured by gender association: feminine word vectors are red, masculine word vectors are blue. We observe migration of some vectors across previous `class divisions', suggesting that \DDDD{} is helping to remedy the bias by neighbour phenomenon.}
    \label{fig:6projections}
\end{figure*}

Our approach builds on HARD-DEBIAS's use of pre-selected {\it instances}, where an instance is a pair of words and their corresponding embedded representations. Each selected pair, for example \textit{her, his}, or \textit{she, he}, defines a direction in the representation space. HARD-DEBIAS takes these directions, summarizes them using PCA to find a single direction, and then projects the representations of all non-gendered words onto its orthogonal complement. In contrast, we apply \DDDD{} to find decision directions, rather than principal components, that separate vectors of masculine- and feminine-labeled words, using the list of masculine and feminine words identified by \cite{Zhao2018}. We then project all words in the embedding orthogonal to these directions. We tested \DDDD{} on both the original Google News embedding using word2vec \citep{Mikolov2013} and on a smaller benchmarking variant (w2vnews) used in related debiasing experiments. In preliminary experiments, we found that 6 iterations of \DDDD{} on the w2vnews dataset led to convergent CV accuracy; hence the choice of 6 iterations for all experiments. 



\subsubsection{Bias By Neighbour} GG observed that the most extreme words at each end of the gender direction (which they consider to be the difference between \textit{she} and \textit{he}) cluster well using standard $k$-means ($k=2$), and that this clustering persists after applying several existing debiasing techniques. To quantify this effect, GG uses cluster-based classification accuracy, calling this quantity `bias by neighbour'. We reproduced the gender direction vector using the w2vnews embedding set provided with HARD-DEBIAS for comparison with \DDDD{}. We then used $k$-means ($k=2$) to cluster the 500 most biased words from each extreme of the gender direction into two clusters. This method matches  gender labels to clusters with 99.8\% accuracy on w2vnews and  with 99.98\% accuracy on Google News. GG reported clustering that matches gender labelling in 92.5\% of cases after HARD-DEBIAS and 85.6\% of cases in GN-GloVe. Using \DDDD{}, projections 2, 4, and 6 achieve reductions of bias by neighbour to 95.9\%, 87.4\%, and 74.3\%, respectively. These results are visualized by Figure \ref{fig:6projections}.



We repeated the experiment using the full Google News embedding to test on a larger set. Using 2, 4, and 6 \DDDD{} iterations yielded accuracies of 71.6\%, 67.4\%, and 68.2\%, respectively, on associating clusters with gender labels. Accuracy did not converge after 6 projections in this embedding, suggesting that additional iterations could more thoroughly remove bias from this larger set of word vectors. 

\subsubsection{Debiasing Professions} Exploring a different manifestation of bias, GG observed that word embeddings of masculine-biased professions cluster well together after debiasing with HARD-DEBIAS and GN-GloVe. Conversely, feminine-biased professions such as \textit{nurse} do not have as many masculine neighbors. We took the list of all profession terms that have a positive dot product with the gender direction and labelled them as `masculine-biased professions'. From this set, we then computed the 100 nearest neighbours for all words in the `profession' data provided in HARD-DEBIAS, and counted the number of such masculine-biased professions in the 100 nearest points in vector space. Changes in masculine nearest neighbour count after 6 iterations of \DDDD{} are visualized by Figure \ref{fig:nearestneighboursprofessions}. 

\begin{figure*}[t]
    \centering
    \begin{subfigure}{.5\textwidth}
        \centering
        \includegraphics[width=1.0\linewidth]{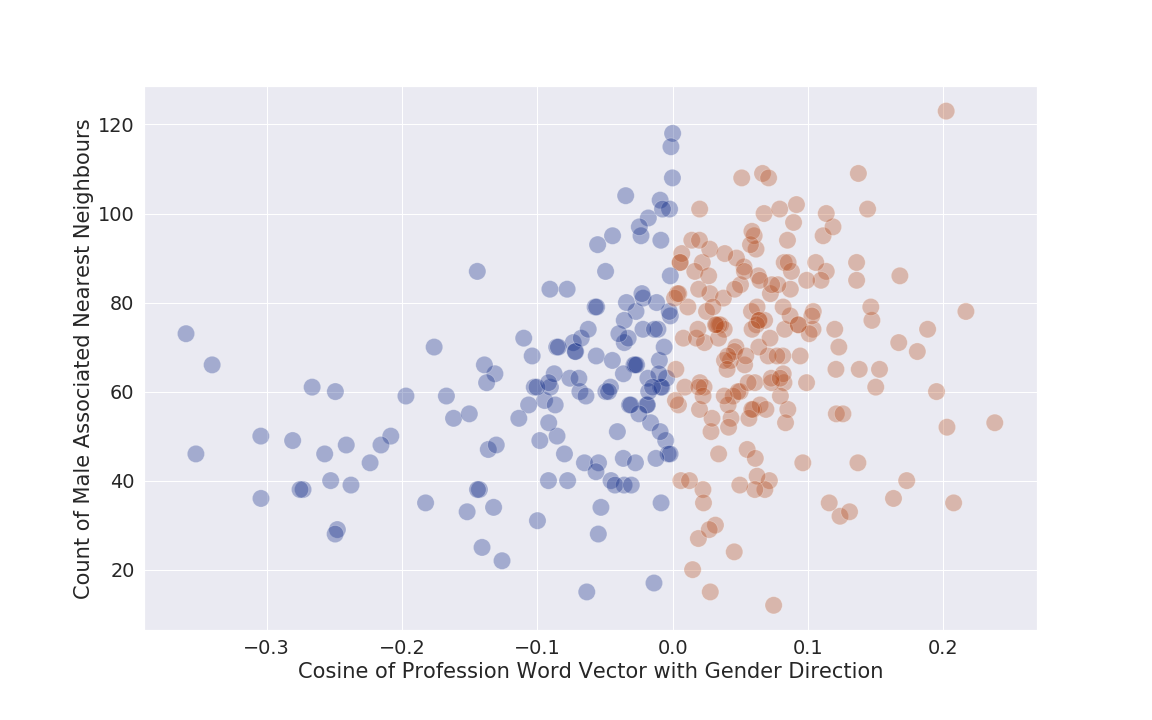}
        \caption{Nearest Neighbour Count Before Debiasing}
        \label{fig:sub3}
    \end{subfigure}%
    \begin{subfigure}{.5\textwidth}
        \centering
        \includegraphics[width=1.0\linewidth]{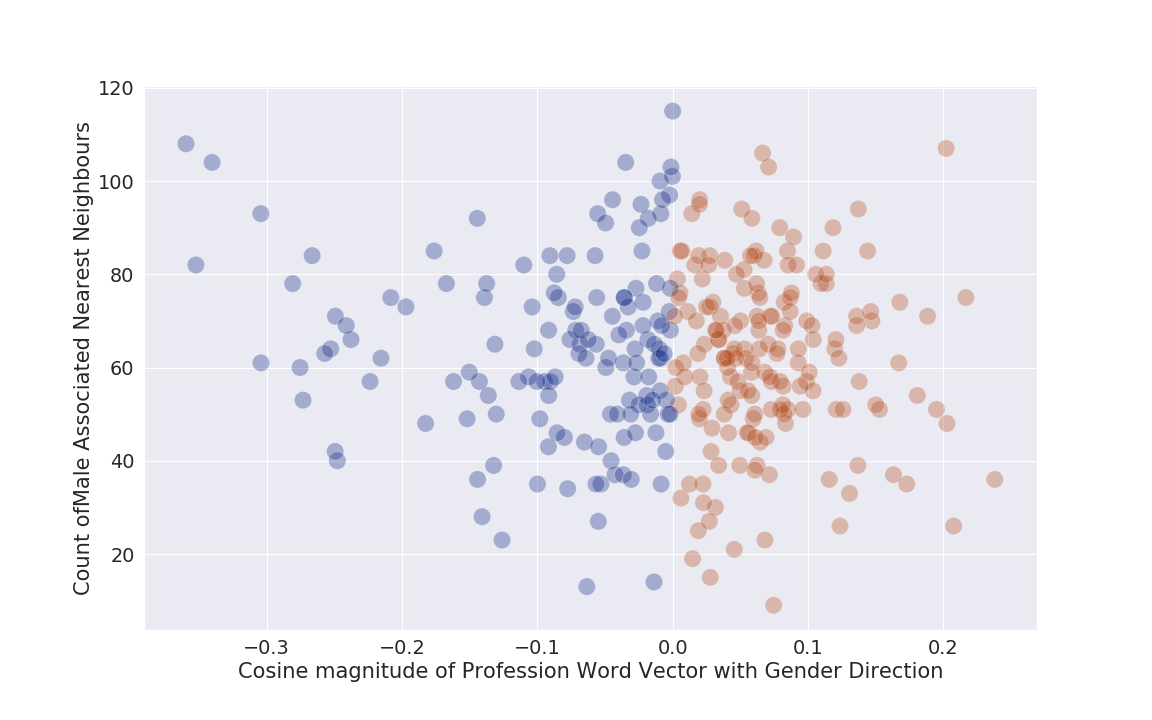}
        \caption{Nearest Neighbour Count After Debiasing with 6 \DDDD{} projections}
        \label{fig:sub4}
    \end{subfigure}
    \caption{For each profession word vector with a gender association, we measure it's dot product with the embedding space's gender direction, and count the number of nearest neighbours in vector space that are masculine oriented professions. Points are coloured by gender association: blue dots correspond to feminine word vectors and red correspond to masculine word vectors. Original counts of nearest neighbours are seen in (a) and after 6 iterations of \DDDD{} in (b). While much of the structure near the `0' of the gender direction is maintained, the most feminine oriented vectors have increased numbers of masculine oriented vectors, suggesting progress towards debiasing.}
    \label{fig:nearestneighboursprofessions}
\end{figure*}

\subsubsection{Recoverability} 
An important aspect of evaluating any debiasing method is to test whether bias can be recovered using other algorithms. To this end and in-line with GG, we estimated bias recoverability using more powerful \textit{non-linear models}, namely radial basis function (RBF) kernel SVMs implemented in scikit-learn with default parameters trained to separate the masculine- and feminine associated word vectors in w2vnews. The resulting accuracy is 59.1\%; the same score we found a linear SVM to converge to when applying \DDDD{}. Running the experiment on the full Google News embedding results in 51.5\% accuracy.  This shows improvement over the scores of 88.88\% from HARD-DEBIAS and 96.53\% from GN-GloVe observed by GG.


\begin{figure*}[t]
    \centering
    \includegraphics[width=460px]{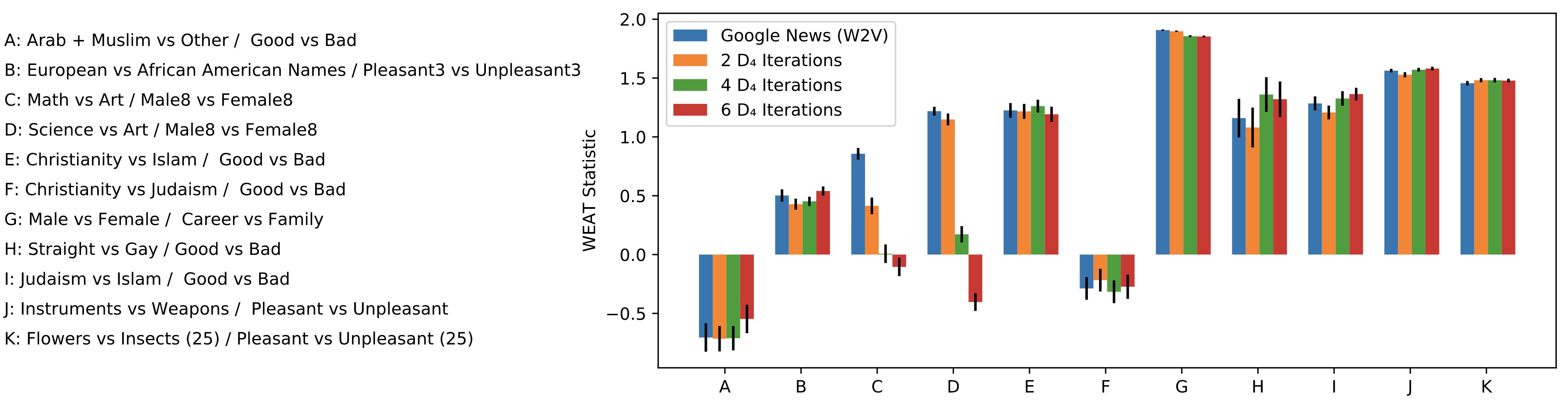}
    \caption{Estimation of bias using the Word Embedding Association Test developed by \citep{Caliskan2017}. Values near zero indicate no association. Estimations were made on the original Google News word embedding, and 2, 4 and 6 \DDDD{} iterations. We observe \DDDD{}'s ability to remove targeted concepts without altering other associations in the embedding.}
    \label{fig:WEAT}
\end{figure*}

\noindent
\newpage
\subsubsection{Word Embedding Associations Test} Last, we used the word embedding association test (WEAT) \cite{Caliskan2017} with the implementation available at \url{https://github.com/hljames/compare-embedding-bias} to further evaluate debiasing with \DDDD{} on the full Google News embedding. We did not use the w2vnews embedding as it was missing many vectors for words in the test. Results are visualized by Figure \ref{fig:WEAT}. We see that some of the debiasing generalizes to appropriate categories in the test, and in some cases results in changes from positive to negative associations. Debiasing with \DDDD{} did not modify the associations in Career vs.\ Family to the same extent as other gendered categories. This suggests that further debiasing could benefit from an expanded set of curated examples to compare against than the ones we used here. Multiple curated lists could be used to sequentially debias embeddings with \DDDD{} or related methods. 

We have demonstrated here that \DDDD{} has potent debiasing capabilities and can be applied post-hoc to data in word embedding vector spaces.

\section{Relationship to Other Methodology}\label{ss:related}

We now discuss how \DDDD{} relates to three areas of methodology: adversarial training, PCA, and kernel methods. 

\subsubsection{Adversarial Training} Adversarial training is used to describe neural network training that seeks to construct a feature representation that is able to learn a target concept but is also \textit{unable} to learn a distractor concept \citep{Goodfellow}. This is undertaken during learning of the representation, where the two types of training are interleaved. The goal is similar to that of \DDDD{} but there are important differences: \DDDD{} creates a representation space orthogonal to the decision direction learned from the specified target. The resulting new representation should perform poorly on the specified target, but there is no specification of a particular learning task that the representation should perform well on. Also, adversarial training can create different nonlinear representations depending on network architecture, whereas \DDDD{} operates only in the linear feature space. (Though it is possible to implement \DDDD{} using kernels.) Finally, \DDDD{} can be applied as a post-hoc step that is computationally inexpensive relative to training or re-training a large neural network.

\subsubsection{Principal Components Analysis} Although both PCA and \DDDD{} both involve projecting the feature matrix onto a linear subspace, \DDDD{} projects the \textit{rows} of $\mX$ onto a linear subspace whereas PCA and its variants project the \textit{columns} of the data matrix onto a lower-dimensional space with an orthogonal representation \citep{jolliffe2016principal}. Iterative methods for PCA, such as Schur-complement deflation, sequentially identify directions onto which the columns of $\mX$ are projected. Considering what happens if instead of projecting the rows of $\mX$ onto the space orthogonal to $\vomega$ as $\DDDD{}$ does, we project the columns of $\mX$ onto the space orthogonal to $\mX\vomega$, we have the following theorem:
\begin{theorem}\label{thm:schur}
For full-rank $\mX$, Schur-complement deflation of its columns onto the orthogonal complement of $\mX\vomega$ and decision-direction deflation onto the orthogonal complement of $\vomega$ are equivalent iff the columns of $\mX$ are orthonormal.
\end{theorem}
\begin{proof}
$\left(\mI - \frac{\mX\vomega\vomega^\T\mX^\T}{\vomega^\T\mX^\T\mX\vomega} \right) \mX = \mX \left(\mI - \vomega\vomega^\T \frac{\mX^\T\mX}{\vomega^\T\mX^\T\mX\vomega}\right).$
If $\mX$ has orthonormal columns, then $\mX^\T\mX = \mI$ and the r.h.s.\ simplifies to projecting the rows of $\mX$ onto the orthogonal complement of $\vomega$. Conversely, if the Schur-complement deflation is equivalent to projection onto the orthogonal complement of $\vomega$, then $\vomega\vomega^\T \frac{\mX^\T\mX}{\vomega^\T\mX^\T\mX\vomega} = \vomega\vomega^\T$. This implies that $\frac{\mX^\T\mX}{\vomega^\T\mX^\T\mX\vomega}$ is idempotent. If $\mX$ has full rank, then $\mX^\T\mX = \mI$ since $\mI$ is the only full-rank idempotent matrix.
\end{proof}
The interesting implication of Theorem~\ref{thm:schur} is that if the features in $\mX$ are \textit{not} orthogonal, then the two projections (rows versus columns) give different results. In particular, this means that projecting the columns of $\mX$ onto the orthogonal complement of $\mX\vomega$ will not in general remove all variability in the direction of $\vomega$, regardless of how $\vomega$ is found.

\subsection{Kernel Methods}

While we envision \DDDD{} to be primarily applicable to explicit feature spaces learned e.g.\ by neural networks, we can perform the same projection steps in the feature space induced by a kernel. In this setting, rather than transforming the original data matrix $\mX$, we transform a kernel matrix $\mK$ after every step to remove variability along chosen directions in the implicit high-dimensional feature space. We note that this is possible without having an explicit representation of the direction, by the representer theorem \citep{shawe-taylor_cristianini_2004}.

\section{Conclusion}\label{ss:conclusion}

We have described a new algorithm, Decision-Directed Data Decomposition, for removing information from a dataset. It is simple and scalable, and can be used for removing a target concept, for improving generalization, and for debiasing. We have shown in particular that it produces state-of-the-art results in word embedding debiasing. As a post-hoc method for debiasing in high-dimensional representations of data, \DDDD{} improves accessibility for users of pre-trained natural language processing and computer vision models who want to address bias in neural representations but may not have access to sufficient computational resources for end-to-end model fitting.

For future method development, we are interested in two approaches for mitigating loss of information by \DDDD{}. The first is to apply the projections in an expanded feature space, either implicitly using the kernel trick together with techniques like SMO \citep{platt1998sequential}, or explicitly using for example random Fourier features \citep{rahimi2008random}. For future applications, \DDDD{} could be used to manipulate contextual word embedding spaces, which increasingly involve neural networks and datasets that are too large for many end users to train directly. These include BERT \citep{devlin-etal-2019-bert} and MEGATRON \citep{shoeybi2019megatronlm}. Future work will focus on the challenges of debiasing morphing word representations, which change depending on context. \cite{karve-etal-2019-conceptor} have laid exciting groundwork in addressing these challenges.


\bibliography{arxiv.bib}

\end{document}